\documentclass[american,a4paper,runningheads,envcountsect]{llncs}

\usepackage{babel}
\usepackage[utf8]{inputenc}
\usepackage[T1]{fontenc}
\usepackage[ruled,vlined,linesnumbered]{algorithm2e}

\usepackage{paralist}
\usepackage{amsmath,amsthm,amssymb}
\usepackage{mathtools}
\usepackage{mathrsfs}
\usepackage{nicefrac}
\usepackage{fca}
\usepackage{dsfont} % For double stroke math digits

\usepackage[colorlinks,citecolor=blue,urlcolor=black,hidelinks,linktocpage]{hyperref}
\usepackage[all]{hypcap}
\usepackage{cleveref}
\let\cref\Cref
\usepackage[babel,kerning=true]{microtype}
\usepackage[subtle]{savetrees}

\newtheorem{problm}{Problem}

\theoremstyle{definition}
\crefname{problm}{Problem}{Problems}

\theoremstyle{remark}

\renewcommand{\epsilon}{\varepsilon}
\renewcommand{\phi}{\varphi}

\newcommand{\Abs}[1]{\left| #1 \right|}

\usepackage{csquotes}
\usepackage{booktabs}
\usepackage{todonotes}
\presetkeys{todonotes}{color=blue!5}{}

\usepackage{fancyvrb}

\usepackage[backend=bibtex,style=numeric-comp,doi=false,isbn=false,%
url=false]{biblatex}

\newcommand{\Scon}{\mathbb{S}}
\newcommand{\Ocon}{\mathbb{O}}
\DeclareMathOperator{\Pclass}{P}

\DeclareMathOperator{\Ext}{Ext}

\DeclarePairedDelimiter\abs{\lvert}{\rvert}%

\usepackage{tabularx, multirow,booktabs}
\begin{document}

\title{Automatic Textual Explanations of Concept Lattices}

\date{\today}

\author{Johannes Hirth \inst{1,3}\orcidID{0000-0001-9034-0321}
  \and Viktoria Horn\inst{2,3}
  \and Gerd Stumme\inst{1,3}
  \and Tom Hanika\inst{1,3,4,5}
}
\institute{%
 Knowledge \& Data Engineering Group,
 University of Kassel, Germany\\[0.5ex]
 \and
 Partizipative IT-Gestaltung,
 University of Kassel, Germany\\[0.5ex]
 \and
 Interdisciplinary Research Center for Information System Design\\
 University of Kassel, Germany\\
 \and
 School of Library and Information Science,\\
 Humboldt-Universität zu Berlin, Berlin, Germany
 \and
 Institute of Computer Science,
 University of Hildesheim, Germany\\
 \email{$\{$hirth,stumme$\}$@cs.uni-kassel.de},
 \email{viktoria.horn@uni-kassel.de},
 \email{tom.hanika@uni-hildesheim.de}
}
\maketitle

% \blfootnote{Authors are given in alphabetical order.
%   No priority in authorship is implied.}

\begin{abstract}
  Lattices and their order diagrams are an essential tool for
  communicating knowledge and insights about data. This is in
  particular true when applying Formal Concept Analysis.  Such
  representations, however, are difficult to comprehend by untrained
  users and in general in cases where lattices are large.  We tackle
  this problem by automatically generating textual explanations for
  lattices using standard scales. Our method is based on the general
  notion of \emph{ordinal motifs} in lattices for the special case of
  standard scales.  We show the computational complexity of
  identifying a small number of standard scales that cover most of the
  lattice structure. For these, we provide textual explanation
  templates, which can be applied to any occurrence of a scale in any
  data domain. These templates are derived using principles from
  human-computer interaction and allow for a comprehensive textual
  explanation of lattices. We demonstrate our approach on the spices
  planner data set, which is a medium sized formal context comprised
  of fifty-six meals (objects) and thirty-seven spices
  (attributes). The resulting 531 formal concepts can be covered by
  means of about 100 standard scales.
\end{abstract}

\keywords{Ordered Sets, Explanations, Formal~Concept~Analysis, Closure~System, Conceptual~Structures}

\section{Introduction} 
There are several methods for the analysis of relational data. One
such method is Formal Concept Analysis~\cite{fca-book} (FCA). The
standard procedure in the realm of FCA is to compute the concept
lattice, i.e., a data representation on the ordinal level of
measurement~\cite{measurement-levels}. Ordered data structures are
comparatively more comprehensible for users than, e.g., Euclidean
embeddings. Nevertheless, untrained users may have difficulties in
grasping knowledge from lattices (and lattice diagrams). Moreover,
even trained users cannot cope with lattice structures of large
size. In addition, there are up until now only rudimentary methods to
derive basic meaning of lattices that are of standard
scale~\cite[Figure 1.26]{fca-book}.

A meaningful approach to cope with both issues is to employ more
complex ordinal patterns, e.g., scales composed from standard
scales. A recent result by \textcite{ordinal-motif} allows for the
efficient recognition of such patterns, there called as \emph{ordinal
  motifs}. Based on these we propose a method to automatically
generate textual explanations of concept lattices.  For the
recognition of ordinal motifs we employ \emph{scale-measures}, i.e.,
continuous maps between closure spaces. These are able to analyze
parts of a conceptual structure with respect to a given set of scale
contexts.  While this approach is very expressive there may be
exponentially many scale-measures. Therefore we introduce an
importance measure of ordinal motifs based on the proportion of the
conceptual structure that they reflect. With this our method can
identify a small number of ordinal motifs that covers most of the
concept lattice. 

An advantage of employing sets of standard scales is their well-known
structural semantic, cf. \emph{basic meaning} Figure 1.26
\cite{fca-book}. Based on this we constructed for every standard scale
textual templates based on principle from human computer
interaction. In detail we applied the five goodness
criteria~\cite{mamun2021assessing} for explainability in machine
learning to ensure that the textual templates are human
comprehensible. 

Besides our theoretical investigations we provide an experimental
example on a real world data set of medium size. All proposed methods
are implemented in conexp-clj~\cite{conexp-clj}, a research tool for
Formal Concept Analysis. Our approach is not only beneficial for
untrained users but also provides explanations of readable size for
concept lattices that are too large even for experienced users.

\section{Formal Concept Analysis}
Throughout this paper we presume that the reader is familiar with
standard FCA notation~\cite{fca-book}. In addition to that, for a
formal context $\context\coloneqq (G,M,I)$ we denote by
$\context{[H,N]}\coloneqq (H,N,I\cap H\times N)$ the induced
sub-context for a given set of objects $H\subseteq G$ and attributes
$N\subseteq M$. If not specified differently the lift of a map
$\sigma: G_1 \to G_2$ on $\mathcal{P}(G_1)\to \mathcal{P}(G_2)$ is
defined as $\sigma(A) \coloneqq \{\sigma(a)\mid a\in A\}$ where
$A\subseteq G_1$. The second lift to
$\mathcal{P}(\mathcal{P}(G_1))\to \mathcal{P}(\mathcal{P}(G_2))$ is
defined as
$\sigma(\mathcal{A}) \coloneqq \{\sigma(A) \mid A\in \mathcal{A} \}$
for $\mathcal{A}\subseteq \mathcal{P}(G_1)$. For a closure system $\mathcal{A}$ on
$G$ we call $\mathcal{D}$ a finer closure system, denoted
$\mathcal{A}\leq\mathcal{D}$, iff $\mathcal{D}$ is a closure system on
$G$ and $\mathcal{A}\subseteq\mathcal{D}$. In this case 
$\mathcal{A}$ is coarser than $\mathcal{D}$. We call $\mathcal{D}$ a
sub-closure system of $\mathcal{A}$ iff $\mathcal{D}$ is a closure
system on $H\subseteq G$ and
$\mathcal{D}=\{H\cap A\mid A\in \mathcal{A}\}$.

Note that there are other definitions for sub-closure systems in the
literature \cite{smeasure}.

\section{Recognizing Ordinal Motifs of Standard Scale}
For the generation of textual explanations we recognize parts of the
concept lattice that match an ordinal motif, i.e., are isomorphic to a
standard scale. For this task we employ (full) scale-measures as
introduced in the following.

\begin{definition}[Scale-Measure (Definition 91 \cite{fca-book})]
  For two formal contexts $\context,\ \Scon$ a map
  $\sigma: G_{\context}\to G_{\Scon}$ is a \emph{scale-measure} iff for all
  $A\in \Ext(\Scon)$ the pre-image $\sigma^{-1}(A)\in
  \Ext(\context)$. A scale-measure is \emph{full} iff
  $\Ext(\context) = \sigma^{-1}(\Ext(\Scon))$.
\end{definition}

We may note that we use a characterization for full scale-measures
(Definition 91 \cite{fca-book}) which can easily be deduced.  The
existence of a scale-measure from a context $\context$ into a scale
context $\Scon$ implies that the conceptual structure of the image of
$\sigma$ in $\Scon$ is entailed in $\BV(\context)$.  Thus, if we are
able to explain $\Scon$ we can derive a \emph{partial} explanation of
$\context$.  In contrast, for full scale-measures we can derive an
\emph{exact} explanation (up to context isomorphism) of $\context$.
Obviously, both scale-measures and full scale-measures differ in their
\emph{coverage} of $\Ext(\context)$, i.e., partial and exact. However,
both morphisms are defined on the entire set of objects $G$ of
$\context$ and are therefore \emph{global scope}. 

Even though global explanations are the gold standard for explainable
artificial intelligence, they often elude from human comprehensibility
due to their size. Therefore we divide the problem of deriving a
single global explanation into multiple local explanations. 
To \emph{locally} describe a part of context $\context$ a
generalization of scale-measures is introduced in
\textcite{ordinal-motif}.

\begin{definition}[Local Scale-Measures \cite{ordinal-motif}]
  For two contexts $\context,\ \Scon$ a map $\sigma:H\to G_{\Scon}$ is
  a \emph{local scale-measure} iff $H\subseteq G_{\context}$ and
  $\sigma$ is a scale-measure from $\context{[H,M]}$ to $\Scon$. We
  say $\sigma$ is \emph{full} iff $\sigma$ is a full scale-measure
  from $\context{[H,M]}$ to $\Scon$.
\end{definition}

Based on this we want to construct in the following templates for
textual explanations. The basis for these templates are standard
scales. Given a context $\context$, a local (full) scale-measure
$\sigma$ of $\context$ into $\Scon$ we can replace every instance of
an object $g\in G_{\Scon}$ in a textual explanation template of
$\Scon$ by its pre-image $\sigma^{-1}(g)\subseteq G$. This yields a
textual explanation of $\context$ with respect to $\sigma$.

For the standard scales, i.e., \emph{nominal}, \emph{ordinal},
\emph{interordinal}, \emph{crown} and \emph{contranominal}, we show
textual templates in \cref{sec:explain-templates}. These are designed
such that they can be universally applied in all settings.
Prior to discussing the textual templates we have to discuss how to
recognize standard scales in a given formal context.
The general ordinal motif recognition problem was introduced in
\textcite{ordinal-motif}. In this work the authors are only concerned
with the recognition of ordinal motifs based on scale-measures into
standard scales. Nonetheless, we recall the general problem for
enumerating scale-measures.

\begin{problm}[Recognizing Ordinal Motifs~\cite{ordinal-motif}]\label{problem:explain}
  Given a formal context $\context$ and an ordinal motif $\context[S]$
  find a surjective map from $\context$ into $\context[S]$
  that is:

  \begin{center}
    \begin{tabular}{>{\raggedright\hspace{0pt}\itshape}p{2cm}|>{\raggedright\hspace{10pt}}p{5cm}l}
      &$\mathrm{global}$&$\mathrm{local}$\\ \midrule
      $\mathrm{partial}$& scale-measure&  local scale-measure\\
      $\mathrm{full}$   &  full scale-measure & local full scale-measure\\
    \end{tabular}
  \end{center}
\end{problm}

\noindent
The underlying decision problem of \cref{problem:explain} has been
proven to be NP-complete~\cite{ordinal-motif}. In a moment we will
investigate a particular instance of this problem for standard
scales. But first we want to give the idea of how the recognition of
standard scales relates to the overall explanation task. 

In practice we consider families of standard scales for investigating
a given formal context $\context$ such that we have explanation
templates for each scale. Thus we have to solve the above problem for
a family of scale contexts $\mathcal{O}$. Moreover, usually we are not
only interested in a single scale-measure into a scale context
$\context[S]$ but all occurrences of them.

Fortunately, for all standard scales but the crown scales is the
existence of local full scale-measures hereditary with respect to
subsets of $H\subseteq G$. For example a context for which there
exists a local full scale-measure $\sigma:H\to G_{\Scon}$ into the
ordinal scale of size three does also allows for a local full
scale-measure into the ordinal scale of size two by restricting
$\sigma$ to two elements of $H$. Thus when enumerating all local full
scale-measures a large number of candidates do not need to be
considered.

Another meaningful restriction for the rest of this work is to
consider local full scale-measures only. Thus our methods focus on
local full explanations (cf. \cref{problem:explain}). Moreover, this
choice allows to mitigate the enumeration of all scale-measures. For a
family of standard scales of a particular type, e.g., the family of
all ordinal scales, let $\context[S]_{n}$ be the scale context of size
$n$.  We thus consider only the local full scale-measures
$\sigma: H\to G_{\Scon_{n}}$ of $\context$ where there is no local
full scale-measure $H\cup\{g\}\to G_{\Scon_{n+1}}$ from $\context$ to
$\context[S]_{n+1}$ with $g\in G, g\not\in H$.  For example, in case
that $H$ is of ordinal scale with respect to $\sigma$ we can infer
that all proper subsets of $H$ are of ordinal scale.  We remind the
reader that we only consider surjective maps
(cf. \cref{problem:explain}).

\begin{proposition}[Recognizing Standard Scales]
  Deciding whether for a given formal context $\context$ there is a full
  scale-measure into either standard scale
  $\context[N]_{n}$, $\context[O]_{n}$, $\context[I]_{n}$,
  $\context[C]_{n}$ or $\context[B]_{n}$ is in
  $\Pclass$.
\end{proposition}
\begin{proof}
  WLoG we assume that $\context$ is clarified. 

  For a contranominal scale $\mathbb{B}_{n}\coloneqq ([n],[n],\neq)$
  every pair of bijective maps $(\alpha:[n]\to[n],\ \beta:[n]\to [n])$
  is a context automorphism of $\mathbb{B}_{n}$ (\cite{fca-book}).
  Thus we can select an arbitrary mapping from $G$ into $[n]$ and
  check if it is a full scale-measure from $\context$ into the
  contranominal scale $\mathbb{B}_{n}$. The verification of full
  scale-measures is in $\Pclass$~\cite{ordinal-motif}. The same
  reasoning applies to nominal scales $\context[N]_{n}\coloneqq ([n],[n],=)$. 

  For ordinal scales we need to verify that for each pair of objects
  their object concepts are comparable. Hence, the recognition for
  ordinal scales is in $\Pclass$.

  For an interordinal scale
  $\mathbb{I}_{n}\coloneqq ([n],[n],\leq) | ([n],[n],\geq)$ we can
  infer from the extents of $\context$ of cardinality two two possible
  mappings $\sigma_{\leq},\sigma_{\geq}$ that are the only candidates
  to be a full scale-measure. For interordinal scales the extents of
  cardinality two overlap on one object each and form a chain. From
  said chain we can infer two order relations of the objects $G$ given
  by position in which they occur in the chain. From the total order
  on $G$ we can infer a mapping $\sigma_{\leq}:G\to [n]$ where the
  objects are mapped according to their position. We can construct
  $\sigma_{\geq}$ analogously by reversing the positions. All maps
  other than $\sigma_{\leq}$ and $\sigma_{\geq}$ would violate the
  extent structure of the chain. For $\sigma_{\leq}$ and
  $\sigma_{\geq}$ we can verify in $\Pclass$ either is a full
  scale-measure. Moreover, the extents of cardinality two can be
  computed in polynomial time using \texttt{TITANIC} or
  \texttt{next\_closure}. Hence, the recognition for interordinal
  scales is in $\Pclass$.

  For crown scales $\context[C]_{n}\coloneqq ([n],[n],J)$ where
  $(a,b)\in J \iff a=b \text{ or } (a,b)=(n,1) \text{ or } b=a+1$ we
  can select an arbitrary object $g\in [n]$ and select repeatedly
  without putting back a different $h\in [n]$ with
  $\{g\}'\cap \{h\}'\neq \{\}$. Starting from $g$ there is a unique
  drawing order. In order to find a full scale-measure we have to find
  an isomorphic drawing order for the elements of $G$ in the same
  manner. From this we can derive a map $G\to [n]$ with respect to the
  drawing order and verify if it is a full scale-measure. The
  computational cost of the drawing procedure as well as the
  verification is in $\Pclass$.
\end{proof}

We may note that our problem setting in the contranominal case is
related but different to the question by \textcite{contranominal} for
the largest contranominal scale of a context $\context$.

Once we can recognize standard scales we are able to provide
contextual explanations that are based on them.
One may extend the set of scale to non-standard scales, yet this may
be computationally intractable if they cannot be recognized in
polynomial time.

While we are able to decide if a context $\context$ is of crown scale,
it is NP-hard to decide if it allows for a surjective scale-measure
into a crown scale of size $\abs{G}$.

\begin{proposition}\label{lemma:n-crown}
  Deciding for a context $\context$ if there is a surjective
  scale-measure into a crown scale of size $\abs{G}$ is NP-hard.
\end{proposition}
\begin{proof}
  To show the NP-hardness of this problem we reduce the Hamilton cycle
  (HC) problem for undirected graphs to it, i.e., for a graph $G$ is
  there a circle(-path) visiting every node of $G$ exactly ones. This
  problem is known to be NP-complete.

  For the reduction, we map the graph $G\coloneqq (V,E)$
  (WLoG $\abs{V}\geq 2$) to a formal context
  $\context\coloneqq (V,{\hat V} \cup E,\in)$ where
  $\hat V \coloneqq \{\{v\}\mid v\in V\}$. This map is polynomial in
  the size of the input. The set of extents of $\context$ is equal to
  ${\hat V} \cup E \cup \{V,\{\}\}$. 
  The context $\context$ accepts a surjective scale-measure into the
  crown scale of size $\abs{G}$ iff there is a sequence of
  extents of cardinality two $A_1,\dots,A_n$ of $\context$ such that
  $(V,\{A_1,\dots,A_n\})\leq G$ is a cycle visiting each object $v\in V$
  exactly ones. This is the case iff $G$ has a Hamilton cycle.
\end{proof}

First experiments~\cite{ordinal-motif} on a real world data set
with 531 formal concepts revealed that the number of local full
scale-measures into standard scales is too large to be humanly
comprehended. Thus we propose in the following section two
importance measures for selection approaches.

\section{Important Ordinal Motifs}\label{sec:explanation-coverage}
Our goal is to cover large proportions of a concept lattice
$\BV(\context)$ using a small set of scale-measures $\mathcal{S}$ into
a given set of ordinal motifs. We say a concept $(A,B)\in \BV(\context)$
is covered by $(\sigma,\context[S])\in \mathcal{S}$ iff it is
reflected by $(\sigma,\context[S])$, i.e., there exists an extent
$D\in \context[S]$ with $\sigma^{-1}(D)=A$.

The above leads to the formulation of the general \emph{ordinal motif
  covering} problem.

\begin{problm}[Ordinal Motif Covering Problem]\label{problem:motif-covering}
  For a context $\context$, a family of ordinal motifs $\mathcal{O}$
  and $k\in \mathbb{N}$, what is the largest number $c\in \mathbb{N}$
  such that there are surjective local full scale-measures
  $(\sigma_{1},\mathbb{O}_{1}),\dots,(\sigma_{k},\mathbb{O}_{k})$ of
  $\context$ with $\mathbb{O}_{1},\dots,\mathbb{O}_{k}\in \mathcal{O}$
  and
  \[\Abs{\bigcup_{1\leq i \leq k} (\varphi_{\context} \circ
    \sigma_{i}^{-1}) \bigl( \Ext(\mathbb{O}_{i}) \bigr) }~=~c\] where
  $\varphi_{\context}$ denotes the object closure operator of
  $\context$. If $\context$ does not allow for any scale-measure into
  an ordinal motif from $\mathcal{O}$ the value of $c$ is $0$.
\end{problm}

We call the set
$\{(\sigma_{1},\mathbb{O}_{1}),\dots,(\sigma_{k},\mathbb{O}_{k})\}$ an
ordinal motif covering of $\context$.

If one is able to find an ordinal motif covering that reflects all
formal concepts of $\context$ we can construct a formal context
$\context[O]$ which accepts a scale-measure $(\sigma,\context[S])$ if
and only if $(\sigma,\context[S])$ is a scale-measure of $\context$.

\begin{proposition}[Ordinal Motif Basis of $\context$]\label{prop:ordinal-motif-basis}
  Let $\context$ be a formal context with object closure operator
  $\varphi_{\context}$ and ordinal motif covering
  $\{(\sigma_1,\context[O]_1),\dots,(\sigma_{k},\context[O]_{k})\}$ that
  covers all concepts of $\context$, i.e.,
  $c=\abs{\BV(\context)}$. Let
  \[\mathbb{O}\coloneqq \mid_{1\leq i\leq k}
    (G, M_{\context[O]_{i}}, I_{\context[O]_{i},\varphi_{\context}}), \text{ with } (g,m)\in
    I_{\context[O]_{i},\varphi_{\context}}\iff g\in \varphi_{\context}\bigl(\sigma^{-1}_{i}(\{m\}^{I_{\context[O]_{i}}})\bigr)\]

  where $|$ is the context apposition. Then is a pair $(\sigma,\Scon)$ a local full
  scale-measure from $\context{[H,M]}$ to $\context[S]$ iff $\sigma$ is
  a local full scale-measure from $\Ocon[H,M_{\Ocon}]$ to
  $\context[S]$. In this case we call $\context[O]$ an \emph{ordinal motif
  basis} of $\context$.
\end{proposition}
\begin{proof}
  We have to show that the identity map is a full scale-measure from
  $\context$ to $\context[O]$. Hence, we need to prove that all
  attribute extents of $\Ocon$ are extents in $\context$
  \cite[Proposition 20]{smeasure} and each extent of $\context$ is an
  extent of $\Ocon$. For an attribute $m\in M_{\Scon_{i}}$ is
  $\{m\}^{I_{\Scon_{i},\varphi}}\in \Ext(\context)$ per
  definition. The second requirement follows from the fact that
  $c=\abs{\BV(\context)}$.
\end{proof}

The just introduced basis is a useful tool when investigating
scale-measures of a context $\context$ given a set of ordinal motifs
$\mathcal{O}$. One can perceive $\mathcal{O}$ as a set of analytical
tools and the existence of $\context[O]$ implies that a found ordinal
motif covering
$\{(\sigma_1,\context[O]_1),\dots,(\sigma_{k},\context[O]_{k})\}$ is
complete with respect to scale-measures of $\context$.

\subsection{Scaling Dimension Complexity}
An interesting problem based on the ordinal motif covering for
(non-local) scale-measures is to determine the smallest number $k$
such that $c=\abs{\BV(\context)}$. This number is also the
\emph{scaling dimension}~\cite{scaling-dimension} of $\context$ with
respect to the family of scale contexts $\mathcal{O}$. Note that the
scaling dimension for a given context $\context$ and family of scales
$\mathcal{O}$ does not need to exist. In the following we recall the
scaling dimension problem in the language scale-measures.

\begin{problm}[Scaling Dimension Problem \cite{scaling-dimension}]
  For a context $\context$ and a family of scale contexts
  $\mathcal{O}$, what is the smallest number $d\in \mathbb{N}$ of scale contexts
  $\mathbb{S}_{1},\dots,\mathbb{S}_{d}\in \mathcal{S}$, if
  existent, such that $\context$
  accepts a full scale-measure into the semi-product
  \[\underset{1\leq i \leq d}{\Semi}\mathbb{S}_i.\]
\end{problm}

The scaling dimension can be understood as a measurement for the
complexity of deriving explanations for a formal context based on
scale-measures and a set of ordinal motifs.
However, determining the scaling dimension is a combinatorial problem
whose related decision problem is NP-complete, as can be seen in the
following.

\begin{theorem}[Scaling Dimension Complexity]
  Deciding for a context $\context$ and a set of ordinal motifs
  $\mathcal{O}$ if the scaling dimension is at most $d\in \mathbb{N}$
  is NP-complete.
\end{theorem}
\begin{proof}
  To show NP-hardness we reduce the recognizing full scale-measure
  problem (RfSM)~\cite{ordinal-motif} to it.

  For two input contexts $\hat\context$ and $\hat\Scon$ of the RfSM
  let context $\context\coloneqq {\hat\context}$. We map
  ${\hat context}$ to $\context$ and ${\hat \context[S]}$ to the set
  of ordinal motifs $\mathcal{O}\coloneqq \{\hat\Scon\}$ and set
  $d = 1$. This map is polynomial in the size of the input.

  If there is a full scale-measure from $\hat\context$ into
  $\hat\Scon$ we can deduce that there is a full scale-measure of
  $\context$ into the semi-product that has only one operand and is
  thus just one element of $\mathcal{O}$. Hence, this element is
  $\hat\Scon$ and therefore the scaling dimension is at most one. The
  inverse can be followed analogously.

  An algorithm to decide the scaling dimension problem can be given by
  non-deterministically guessing $d$ scale contexts
  $\context[S]_{1},\dots,\context[S]_{d}\in \mathcal{O}$ and $d$
  mappings from $\sigma_{i}=G_{\context}\to G_{\Scon_{i}}$. These are
  polynomial in the size of the input. The verification for
  full scale-measures in $\Pclass$~\cite{ordinal-motif}.
\end{proof}

% Based on the construction proposed in \cref{prop:ordinal-motif-basis}
% we can derive that for a context $\context$ is the number $c$ of local
% full scale-measures with $c=\abs{\BV(\context)}$ is an upper bound to
% the scaling dimension. An even smaller upper bound can be derived by
% identifying the smallest number of local full scale-measures that
% cover all meet-irreducible extents.

\subsection{Ordinal Motif Covering with Standard Scales}
The ordinal motif covering problem is a combinatorial problem which is
computationally costly even for standard scales. Thus, we propose in
the following a greedy approach which has two essential steps. First,
we compute all full scale-measures $\mathcal{S}$ which is
computationally tame due to the heredity of local full scale-measures
for standard scales. Our goal is now to identify in a greedy manner
elements of $\mathcal{S}$ that increase $c$ of the ordinal motif
covering the most. Thus, secondly, we select $k$ full scale-measures where
at each selection step $i$ with $1\leq i\leq k$ we select a
scale-measure $(\sigma,\context[O])\in \mathcal{S}$ that maximizes
\cref{eq:1}.

\begin{equation}
  \label{eq:1}
  \Abs{(\varphi_{\context} \circ \sigma^{-1}) \bigl(\Ext(\context[O]) \bigr)
    \setminus \bigcup_{1\leq j < i} (\varphi_{\context} \circ
    \sigma_{j}^{-1}) \bigl( \Ext(\context[O]_{j}) \bigr)  }
\end{equation}
In the above equation $(\sigma_{j},\context[O]_{j})$ denotes the
scale-measure that was selected at step $j\leq i$. The union is the
covering number $c$ of the ordinal motif covering
$(\sigma_1,\context[O]_1),\dots,(\sigma_{i-1},\context[O]_{i-1})$.
Overall, the computed cardinality is equal to the number of concepts
reflected by $(\sigma,\context[O])$ that are not already reflected by
$(\sigma_1,\context[O]_1),\dots,(\sigma_{i-1},\context[O]_{i-1})$.

For obvious reasons this approach results in the selection of
scale-measures that have the largest number of (so far) uncovered
concepts. A downside of this heuristic is that it favors ordinal
motifs that have in general more concepts, e.g., contranominal scales
over ordinal scales. To compensate for this we propose to normalize
the heuristic by the number of concepts of the ordinal motif, i.e.,
$\Abs{\sigma^{-1}\bigl(\Ext(\context[O]) \bigr)}$.

In the first step, the normalized heuristic does not account for the
total size of the ordinal motif. The first selected scale-measure
covers at least the top extent, i.e., $G$, and thus the scores for all
following ordinal motifs are at most
$\nicefrac{\abs{\Ext(\Scon)} - 1}{\abs{\Ext(\Scon)}}$.

\section{Human-Centered Textual Explanations}\label{sec:explain-templates}
We want to elaborate on textual explanations of concept lattices based
on principles drawn from human-computer interaction for state of the
art \emph{human-centered explanations}. One of the currently most
applied fields of these explanations in computer science is Explainable AI
(XAI)~\cite{schwalbe2023comprehensive}.  Developing explainable systems commonly begins with ``an assertion about what makes for a good explanation'' \cite{mueller2021principles}, which are not seldomly based on guidelines or collections of principles. Those principles aim to derive \emph{human-centered textual explanations}
that impart complex concepts in a manner that is accessible, relevant,
and understandable. They are designed to cater to the individual
cognitive and emotional needs of readers, anticipating their concerns
and queries.  Thereby they aim at fostering the understanding of the reader
by exposing reasoning and additional information to accompany data
structures they rely on \cite{tintarev2010designing}.  Moreover,
textual explanations based on goodness criteria in the context of
computer-generated knowledge and information help to strengthen trust
in the computed reasoning results \cite{mamun2021assessing}.

\textcite{mamun2021assessing} proposed five goodness criteria for
explainability in the context of machine learning models.  We identify
them as adaptable to our task for textual explanations of concept
lattices.
The first criterion is \emph{accuracy}, which requires that an
explanation is a valid reflection of the underlying data.
\cite{papenmeier2019model}.
The second criterion is \emph{scope}, which refers to the level of
detail in the explanation, which can vary from explaining a single
action to a global description of a system, depending on the tasks and
needs of the reader.
The third criterion relates to the type of question the explanation
answers, which is called the \emph{explanation form criterion}. The
questions can be of type ``what\ldots'', ``why\ldots'', ``why
not\ldots'', ``what if\ldots'', or ``how to\ldots''.  This is related
to the so-called explanation triggers identified by
\textcite{mueller2019explanation}. In their study,
\textcite{mamun2021assessing} found that many explanations in
Explainable AI contexts were ``what'' statements.
The fourth criterion is \emph{simplicity}, which emphasizes the
importance of making an explanation easy to read and understand (e.g.,
\textcite{kulesza2015principles}). \textcite{mamun2021assessing}
suggested testing the appropriate readability level by comparing the
grade level of other related content with one's explanations.
Finally, the fifth criterion is the \emph{knowledge base criterion},
which emphasizes the importance of providing workable knowledge in the
explanation. Thus, explanations should predominantly be written as
factual statements \cite{mamun2021assessing}.
In the following, we first propose our textual explanation templates
for standard scales and afterwards discuss how the principles above
are implemented in their design.

\begin{description}
\item[Nominal Scale:] The elements $n_1,\dots,n_{k-1}$ and $n_k$ are
  incomparable, i.e., all elements have at least one property that the other
  elements do not have.
\item[Ordinal Scale:] There is a ranking of elements
  $n_1,\dots,n_{k-1}$ and $n_k$ such that an element has all the
  properties its successors has.
\item[Interordinal Scale:] The elements $n_1,\dots,n_{k-1}$ and $n_k$
  are ordered in such a way that each interval of elements has a
  unique set of properties they have in common.
\item[Contranominal Scale:] Each combination of the elements
  $n_1,\dots,n_{k-1}$ and $n_k$ has a unique set of properties they
  have in common.
\item[Crown Scale:] The elements $n_1,\dots,n_{k-1}$ and $n_k$ are
  incomparable. Furthermore, there is a closed cycle from $n_1$, over
  $n_2,\dots n_{k-1}$ and $n_k$ back to $n_1$ by pairwise shared
  properties.
\end{description}

We motivate how our setup based on scale-measures
relates to the goodness criteria above.

\begin{description}
\item[Accuracy] The generation of textual explanations are based on ordinal
  motif coverings with scale-measures, i.e., continuous maps
  between closure spaces. These maps do not introduce any conceptual
  error~\cite{conceptual-scaling-error}.  Moreover, ordinal motif
  coverings can function as a basis for the complete conceptual
  structure of the data set with respect to
  \cref{prop:ordinal-motif-basis}.  Therefore an accurate mapping of
  an explanation onto the represented information is guaranteed.
\item[Scope] For the scope of the introduced explanations we differed
  between global and local explanations which is determined by the
  choice of scale-measures, i.e., local vs non-local.  In addition to
  that we can differentiate between two kinds of \texttt{coverage},
  i.e., full and non-full scale-measures.  However, with our
  experiments and the ordinal motif covering we focus mainly on local
  full explanations.

  Altogether, we can serve different task requirements with the
  explanations.
\item[Explanation Form] The main question addressed by ordinal motifs
  is dependent on the type of scale-measure. For full scale-measures
  we answer the question on \emph{``What is \emph{the} conceptual
    relation between a given set of objects.''} and for non-full
  scale-measures we answer \emph{``What is \textbf{a} conceptual
    relation between a given set of objects.''}.
\item[Simplicity] The presented explanations are written using terms
  familiar for readers with basic knowledge about graphs and
  mathematical descriptions.  Formulations that require prior
  knowledge about conceptual structures have been avoided. In addition
  to that, the textual structure is kept simple and explanations are
  composed of at most two short sentences.
\item[Knowledge Base] The generated textual explanations describe the
  conceptual relations between objects and can thus be considered to
  be factual statements.
\end{description}

All proposed textual explanations are designed to be applicable in
every data domain that is representable by formal contexts. However,
different data domains and applications come with different
requirements for the design of human-centered textual
explanations. Thus, a development of domain specific explanations for
a large variety of settings is advisable.
Given more general principles of HCI~\cite{chao2009human}, user
studies with the prospective users of a system are the gold standard
in evaluating any kind of interaction~\cite{mamun2021assessing}. Since
the focus of this work is to introduce the theoretical foundation on
how to derive human-centered explanations we deem the execution of a
user study future work.

\section{Application Example}\label{sec:experiment}
To show the applicability of our method, we compute the ordinal motif
covering for the \emph{spices planner} data set~\cite{pqcores}. This
context contains fifty-six meals as objects and thirty-seven spices
and food categories as attributes.  The context has 531 formal
concepts for which we found over ten-thousands local full
scale-measures into standard scales.In \cref{tab:identify-motifs} we
recall results~\cite{ordinal-motif} on how many local full
scale-measures there are per family of standard scales.  The most
frequent ordinal motif of the spices planner context is the
interordinal motif. The motif having the largest scale size is the
nominal scale motif, which includes up to nine objects.  There are no
non-trivial ordinal scale motifs in the spices planner context, i.e.,
the size of all local full scale-measure domains into ordinal scales
within the spices planner context is one.  Therefore we exclude the
ordinal scales from the following analysis.

\begin{table}
  \centering
  \caption{Results for ordinal motifs~\cite{ordinal-motif} of the 
    spices planner context. Every column represents ordinal motifs of
    a particular standard scale family.  Maximal lf-sm is the number
    of local full scale-measures for which there is no lf-sm with a
    larger domain.  Largest lf-sm refers to the largest domain size
    that occurs in the set of local full scale-measures. }
  \begin{tabular}{|l|r|r|r|r|r|}
    \hline
    &nominal&ordinal&interordinal&contranominal&crown\\\hline \hline
    local full sm&2342&37&4643&2910&2145\\ \hline
    maximal lf-sm&527&37&2550&1498&2145\\ \hline
    largest lf-sm&9&1&5&5&6\\ \hline
  \end{tabular}
  \label{tab:identify-motifs}
\end{table}

In our experiment we applied the introduced greedy strategy. In
\cref{fig:coverage} we report the extent sizes of selected ordinal
motifs.  In the left diagram we depict in the abscissa the steps of
the greedy selection and in the ordinate the number of newly covered
concepts. We report the results for the standard scales individually
and combined, for the later we also experimented with the normalized
heuristic. In the right diagram we depict the accumulated values,
i.e., the value $c$.  First we observe that the normalized heuristic
does not decrease monotonously in contrast to all other results.
From the right diagram we can infer that the crown, interordinal
and nominal are unable to cover all extents.
The contranominal and the combined scale family took the fewest
selection steps to achieve complete extent coverage.
This followed by the normalized heuristic on the combined scale family
which about thirty percent more steps.
Out of the other scale families the crown scales achieved the highest
coverage followed by the interordinal and nominal scales.

\begin{figure}
  \centering
  \includegraphics[width=0.49\linewidth,trim={1cm 0cm 2cm 2cm},clip]{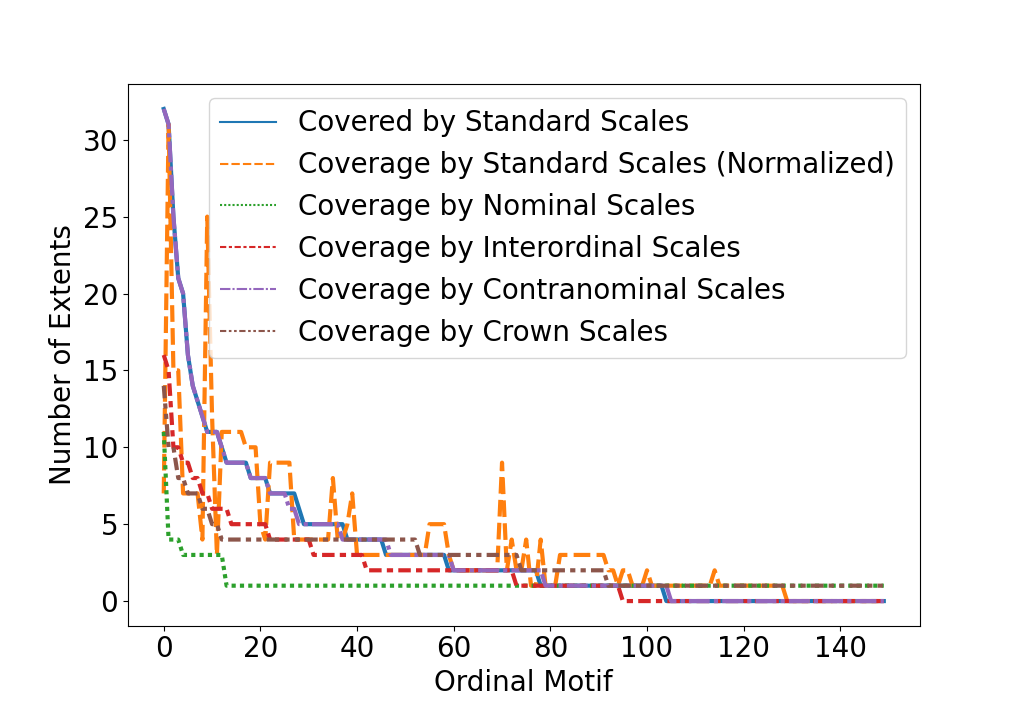} 
  \includegraphics[width=0.49\linewidth,trim={0.8cm 0cm 2cm 2cm},clip]{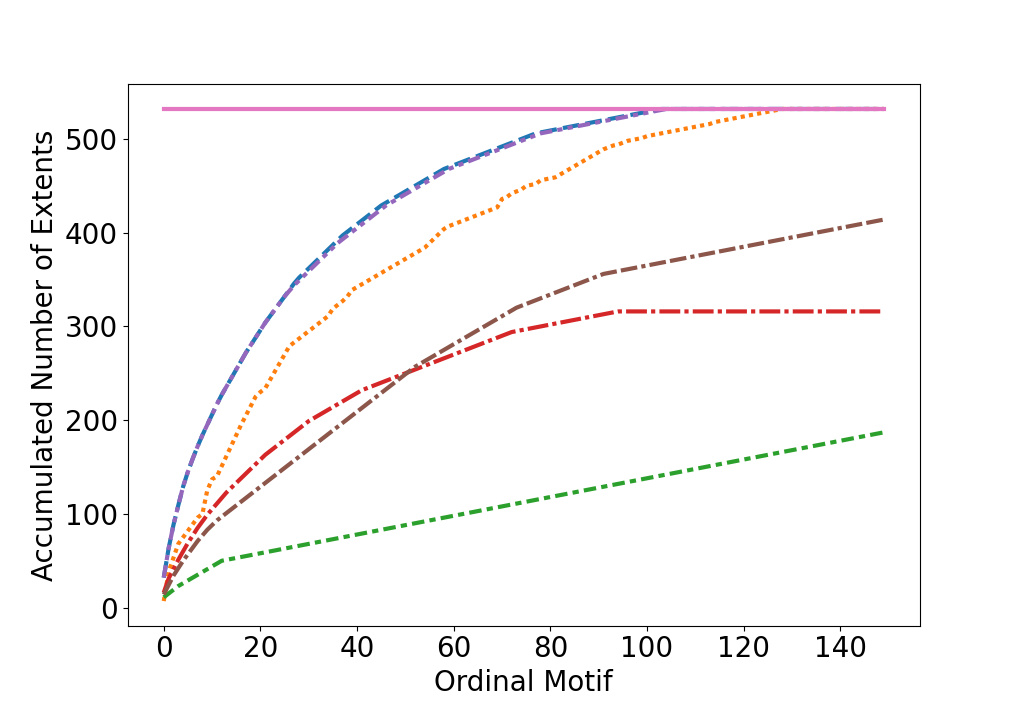}   
  \caption{The extent coverage (left) for the ordinal motif covering
    computation for all and each standard scale family
    individually. The right diagram displays the accumulated coverage
    at each step in the ordinal motif covering computation. The legend
    of the left diagram does also apply to right diagram with the
    addition of the total number of extents (pink) in the context.}
  \label{fig:coverage}
\end{figure}

With \cref{fig:ratio} we investigate the influence of the
normalization on the greedy selection process.  For this we depict the
relative proportion of selected scale types up to a step $i$
(abscissa).  The left diagram shows the proportions for the standard
heuristic and the right reports the proportions for the normalized
heuristic.  We count ordinal motifs that belong to multiple standard
scale families relatively. For example we count the contranominal
scale of size three half for the crown family.  In the first diagram
we see that a majority of the selected ordinal motifs are of
contranominal scales.  This is not surprising since they have the most
concepts among all standard scales.  The interordinal and crown scales
are almost equally represented and the nominal motifs are the least
frequent. In contrast to this the normalized heuristic selects crown
and interordinal motifs more frequently (right diagram).

Overall we would argue that while the normalized heuristic produces
slightly worse coverage scores they provide a more diverse selection
in terms of the standard scales. Therefore, the normalized heuristic
may result in potentially more insightful explanations.

\begin{figure}
  \centering
  \includegraphics[width=0.49\linewidth,trim={1cm 0cm 2cm 2cm},clip]{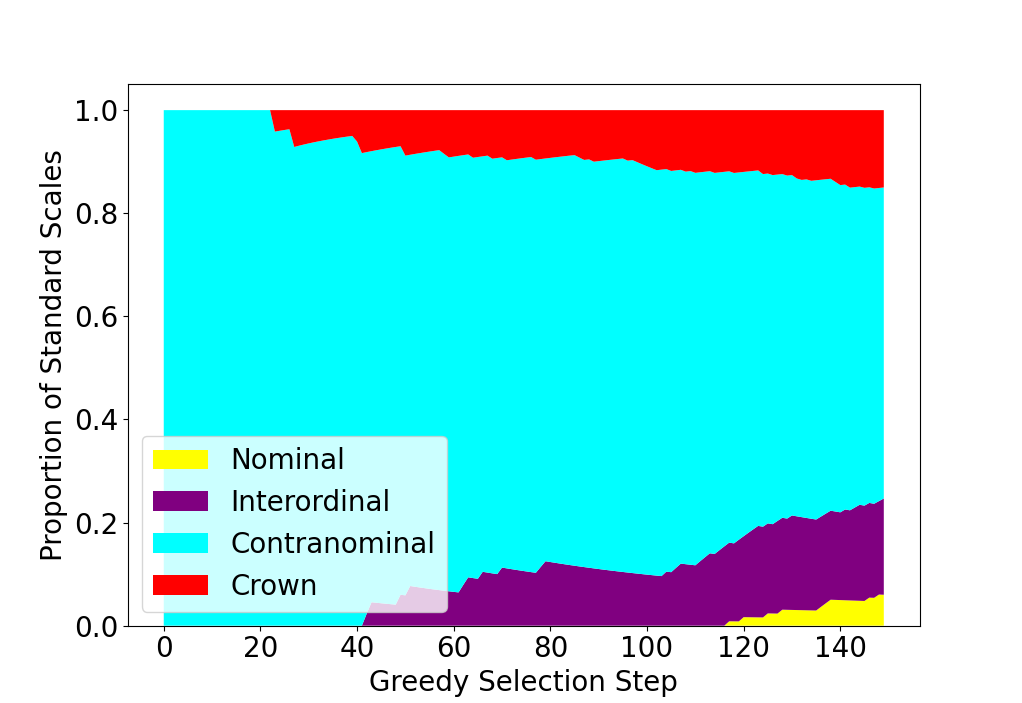} 
  \includegraphics[width=0.49\linewidth,trim={1cm 0cm 2cm 2cm},clip]{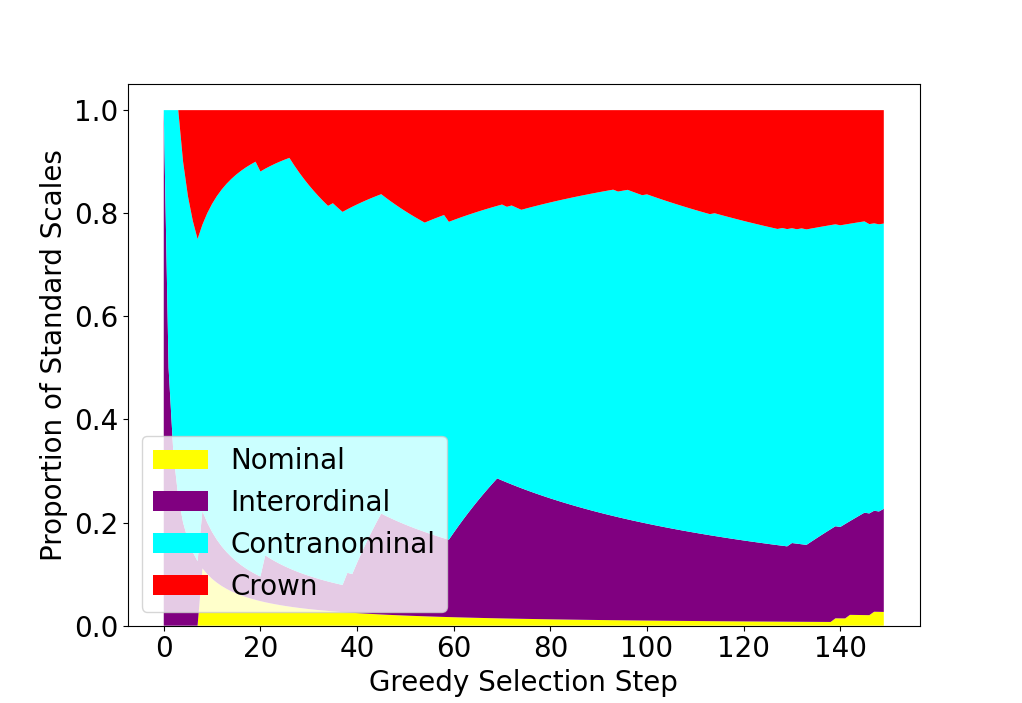}   
  \caption{The ratio of each standard scale family in the ordinal motif covering computation for the standard (left) and normalized heuristic.}
  \label{fig:ratio}
\end{figure}

We conclude by providing automatically generated textual explanations
for spices planner context. For this we report the top ten selections
for the standard and normalized heuristic. First we depict the
explanations for the standard heuristic which consist solely of
contranominal motifs. Thereafter we will turn to the normalized
heuristic results.

\begin{enumerate}[1.]
\item Each combination of the elements \emph{Thyme}, \emph{Sweet
    Paprika}, \emph{Oregano}, \emph{Caraway} and \emph{Black Pepper} has
  a unique set of properties they have in common.
\item Each combination of the elements \emph{Curry}, \emph{Garlic},
  \emph{White Pepper}, \emph{Curcuma} and \emph{Cayenne Pepper} has a
  unique set of properties they have in common.
\item Each combination of the elements \emph{Paprika Roses},
  \emph{Thyme}, \emph{Sweet Paprika}, \emph{White Pepper} and
  \emph{Cayenne Pepper} has a unique set of properties they have in
  common.
\item Each combination of the elements \emph{Paprika Roses},
  \emph{Thyme}, \emph{Allspice}, \emph{Curry} and \emph{Curcuma} has a
  unique set of properties they have in common.
\item Each combination of the elements \emph{Thyme}, \emph{Basil},
  \emph{Garlic}, \emph{White Pepper} and \emph{Cayenne Pepper} has a
  unique set of properties they have in common.
\item Each combination of the elements \emph{Tarragon}, \emph{Thyme},
  \emph{Oregano}, \emph{Curry}, and \emph{Basil} has a unique set of
  properties they have in common.
\item Each combination of the elements \emph{Vegetables},
  \emph{Caraway}, \emph{Bay Leef} and \emph{Juniper Berries} has a
  unique set of properties they have in common.
\item Each combination of the elements \emph{Meat}, \emph{Garlic},
  \emph{Mugwort} and \emph{Cloves} has a unique set of properties they
  have in common.
\item Each combination of the elements \emph{Oregano}, \emph{Caraway},
  \emph{Rosemary}, \emph{White Pepper} and \emph{Black Pepper} has a
  unique set of properties they have in common.
\item Each combination of the elements \emph{Curry}, \emph{Ginger},
  \emph{Nutmeg} and \emph{Garlic} has a unique set of properties they
  have in common.
\end{enumerate}

These explanations cover a total of 195 concepts out of 531.  An
interesting observation is that explanation number eight has only four
objects compared to the five objects of explanation number nine. 
Yet, explanation eight was selected first.
The reason for this is that number eight has more non-redundant concepts
with respect to the previous selections.

The results for the normalized heuristic are very different compared
to the standard heuristic.  The ten selected motifs cover a total of
125 concepts. They consist of one interordinal motif, four
contranominal, one nominal and four motifs that are crown and
contranominal at the same time. For the ordinal motifs that are of
crown and contranominal scale we report explanations for both.

\begin{enumerate}[1.]
\item The elements \emph{Thyme}, \emph{Caraway} and \emph{Poultry} are
  ordered in such a way that each interval of elements has a unique set of
  properties they have in common. 
\item Each combination of the elements \emph{Curry}, \emph{Garlic},
  \emph{White Pepper}, \emph{Curcuma} and \emph{Cayenne Pepper} has a
  unique set of properties they have in common.
\item Each combination of the elements \emph{Allspice},
  \emph{Ginger},\emph{Mugwort} and \emph{Cloves} has a unique set of
  properties they have in common.
\item  Each combination of the elements \emph{Sweet Paprika},
  \emph{Oregano}, \emph{Rosemary} and \emph{Black Pepper} has a unique
  set of properties they have in common.
\item Each combination of the elements \emph{Sauces}, \emph{Basil} and
  \emph{Mugwort} has a unique set of properties they have in common.\\
  The elements \emph{Basil}, \emph{Sauces} and \emph{Mugwort} are
  incomparable. Furthermore, there is a closed cycle from \emph{Basil}
  over \emph{Sauces} and \emph{Mugwort} back to \emph{Basil} by
  pairwise shared properties.
\item  Each combination of the elements \emph{Paprika Roses}, \emph{Meat} and
  \emph{Bay Leef} has a unique set of properties they have in common.\\
  The elements \emph{Paprika Roses}, \emph{Meat} and \emph{Bay Leef}
  are incomparable. Furthermore, there is a closed cycle from
  \emph{Paprika Roses} over \emph{Meat} and \emph{Bay Leef} back to
  \emph{Paprika Roses} by pairwise shared properties.
\item  Each combination of the elements \emph{Saffron}, \emph{Anisey} and
  \emph{Rice} has a unique set of properties they have in common.\\
  The elements \emph{Saffron}, \emph{Anisey} and \emph{Rice}
  are incomparable. Furthermore, there is a closed cycle from
  \emph{Saffron} over \emph{Anisey} and \emph{Rice} back to
  \emph{Saffron} by pairwise shared properties.
\item Each combination of the elements \emph{Vegetables},
  \emph{Savory} and \emph{Cilantro} has a unique set of properties
  they have in common.  \\
  The elements \emph{Savory}, \emph{Cilantro} and \emph{Vegetables}
  are incomparable. Furthermore, there is a closed cycle from
  \emph{Savory} over \emph{Cilantro} and \emph{Vegetables} back to
  \emph{Savory} by pairwise shared properties.
\item The elements \emph{Tarragon}, \emph{Potatos} and \emph{Majoram}
  are incomparable, i.e., all elements have at least one property that
  the other elements do not have.
\item Each combination of the elements \emph{Paprika Roses},
  \emph{Thyme}, \emph{Sweet Paprika}, \emph{White Pepper} and
  \emph{Cayenne Pepper} has a unique set of properties they have in
  common.
\end{enumerate}

\section{Conclusion}
To the best of our knowledge our presented method is the first
approach for the automatic generation of textual explanations of
concept lattices. It is a first step towards making Formal Concept
Analysis accessible to users without prior training in mathematics.
Our contribution comprises the theoretical foundations as well as the
preparation of human-centered textual explanations for ordinal motifs
of standard scale.

In particular, we have shown that the recognition of standard scales
can be done in polynomial time in the size of the context. This is
also the case when the standard scale has exponential many concepts.
This is a positive result for the generation of textual explanations of
large real world data sets.

Based on ordinal motif coverings we are able to limit the generated
textual explanations to a low number of non-redundant conceptual
relations. In detail, we proposed a greedy method for the computation
of ordinal motif coverings based on two heuristics.

To asses the complexity of potential textual explanations of a concept
lattice, we showed the relation between ordinal motif coverings and
the scaling dimension. For the later we proved that the computational
complexity of the related decision problem is NP-complete.

Accompanying our theoretical investigation, we derived criteria on how
to derive textual explanations for ordinal motifs with principles from
human-computer interaction. In addition to that, we demonstrated the
applicability of our approach based on a real world data set. 

As a next logical step, we envision a participatory user study. This
will lead to improved textual explanations for ordinal motifs that are
easier to comprehend by humans. Moreover, the development of domain
specific textual explanations may increase the number of applications for
our proposed methods.

% \section*{Declarations}
% \textbf{Funding and/or Conflicts of interests/Competing interests:}
% This work was funded by the German Federal Ministry of Education and
% Research (BMBF) in its program ``FAIRDIENSTE'' under grant number
% 16KIS1249K. The authors have no competing interests to declare that
% are relevant to the content of this article.

\printbibliography

\end{document}